\documentclass{article}
\usepackage{graphicx}
\usepackage{amsfonts}
\usepackage{amsthm}
\title{Quadratically constrained quadratic programming for classification using particle swarms and applications
}

\author{Deepak Kumar, A G Ramakrishnan\\ 
Medical Intelligence and Language Engineering (MILE) Laboratory,\\Department of Electrical Engineering,\\Indian Institute of Science, Bangalore - 560012, India\\ Tel.: +91-80-22932935\\Fax: +91-80-23600444\\
Email:dipkmr@gmail.com,ramkiag@ee.iisc.ernet.in }

\date{}

\begin{document}

\maketitle

\begin{abstract}
Particle swarm optimization is used in several combinatorial optimization problems. In this work, particle swarms are used to solve quadratic programming problems with quadratic constraints. The approach of particle swarms is an example for interior point methods in optimization as an iterative technique. This approach is novel and deals with classification problems without the use of a traditional classifier. Our method determines the optimal hyperplane or classification boundary for a data set. In a binary classification problem, we constrain each class as a cluster, which is enclosed by an ellipsoid. The estimation of the optimal hyperplane between the two clusters is posed as a quadratically constrained quadratic problem. The optimization problem is solved in distributed format using modified particle swarms. Our method has the advantage of using the direction towards optimal solution rather than searching the entire feasible region. Our results on the Iris, Pima, Wine, and Thyroid datasets show that the proposed method works better than a neural network and the performance is close to that of SVM. 

\end{abstract}

\section{Introduction}
\label{intro}

A class of algorithms originated for minimizing or maximizing a function $f(\mathbf{x})$, while satisfying some constraints $g(\mathbf{x})$. In the history of optimization, the function $f(\mathbf{x})$ and the constraints $g(\mathbf{x})$ were linear and the problem was known as linear programming (LP). One of the early published algorithms for solving the linear programming was given by Dantzig, popularly known as Simplex method \cite{dantzig}. As the number of dimensions and constraints increased, solving the LP using simplex method became hard. The inability of simplex method was that it could not solve LP in polynomial time. Khachiyan proposed ellipsoid algorithm as an alternative to simplex method and proved that it could reach the solution iteratively in polynomial time \cite{khachiyan}. The practical infeasible condition of ellipsoid algorithm led to the evolution of  several interior or barrier point methods. One of the well known interior point methods is Karmarkar's method proposed by Narendra Karmarkar \cite{karmarkar}. 


Binary classification is one of the active research areas in machine learning \cite{bishop,duda}. There are several ways to train a binary classifier. The class labels of a data set can be stored and retrieved during classification using the approach of nearest neighbor \cite{derrac}. A hyperplane is learnt for classification by training a neural network, which may not always be optimal \cite{bishop2}. Vapnik and others formulated the problem of classification as optimization. This method is known as support vector machines (SVMs) \cite{vapnik}. Sequential minimal optimization (SMO) is a technique which solves the optimization problem in SVMs \cite{platt}. Decision trees, Bagging, and Boosting techniques are also used in binary classification \cite{duda,fernandez,galar,lopez}.


\subsection{Motivation}
Nearest neighbor method does not involve any modeling to reduce the storage of the training data set \cite{derrac}. On the other hand, neural network and SVM model the data with an objective function to estimate a hyperplane, which is used in classification. In neural network approach, the objective function is a least squares, which is quadratic in nature and is minimized for the given data set. The hyperplane obtained from a neural network may or may not be optimal, since it depends on the number of layers and weights used to train the network. SVM uses quadratic programming formulation with linear constraints for minimizing the objective function \cite{gonzalez2013}. Several variants of Even though the objective function used in neural network or SVM is a quadratic programming problem, the constraints are linear. 

If there is a way to model linear constraints as quadratic constraints, then the objective function becomes quadratically constrained quadratic programming (QCQP). In this paper, binary classification is posed as a QCQP problem and a novel solution is proposed using particle swarm optimization (PSO). One of the advantages of this approach is that it solves the QCQP problem without the need for gradient estimation.



The paper is organized as follows: QCQP and PSO are described in the background section and the solution for quadratically constrained quadratic programming using particle swarms is described in Sec. 3. The proposed method is compared with Khachiyan's and Karmarkar's algorithms for linear programming and with neural networks and SVM for quadratic programming in Sec. 4 on experiments and results. Section 5 concludes the paper with suggestions for future work.

\section{Background}
The formulation of QCQP and PSO are described in the subsections below. 

\subsection{Quadratically Constrained Quadratic Programming}
A general quadratically constrained quadratic programming problem is expressed as follows \cite{bomze2,bomze,boyd}:
\begin{equation}
\begin{array}{c c c}
minimize & f(\mathbf{x})~=~\mathbf{x}^TP_0\mathbf{x} + 2q_0\mathbf{x} + r_0 & \\
subject~to & \mathbf{x}^TP_i\mathbf{x} + 2q_i\mathbf{x} + r_i~\geq~0 & \forall i \in 1,2...m
\end{array}
\end{equation} 
where $\mathbf{x}~\in~\mathbb{R}^n$. 
If $P_i~$ are positive semi-definite $~\forall~i~\in~0,1,2...m$, then the problem becomes convex QCQP.



\subsection{Particle Swarm Optimization}
Particle swarm optimization was proposed for optimizing in the weights space of a neural network \cite{kennedy}. PSO has been applied to numerous applications for optimizing non-linear functions \cite{kennedybook,spadoni,zhan}. PSO evolved by simulating bird flocking and fish schooling. The advantages of PSO are that it is simple in conception and easy to implement. Particles are deployed in search space and each particle is evaluated against an optimization function. The best particle is chosen as a directing agent for the rest. The velocity of each particle is controlled by both the particle's personal best and the global best. During the movement of the particles, a few of them may reach the global best. The movement of particles is adapted from the genetic algorithms or evolutionary programming.

Let \textbf{X} = $\{\mathbf{x}_1, \mathbf{x}_2, ... \mathbf{x}_k\}$ be the particles deployed in the search space of the optimization function, where $k$ is the number of particles and \textbf{V} = $\{\mathbf{v}_1, \mathbf{v}_2, ... \mathbf{v}_k\}$ are the velocities of the respective particles. $\mathbf{x}_i,~\mathbf{v}_i~\in~\mathbb{R}^n$ for all the $k$ particles. A simple PSO update is as follows,
\begin{itemize}
\item Velocity update equation \\
\begin{equation}
\mathbf{v}_i^j = w\mathbf{v}_i^{j-1} + c_1r_1~(\mathbf{x}_{bi} - \mathbf{x}_i^{j-1}) + c_2r_2(\mathbf{x}_{bg} - \mathbf{x}_i^{j-1})
\end{equation}
where $w$ is the weight for the previous velocity; $c_1$, $c_2$ are constants and $r_1$, $r_2$ are random values varied in each iteration. $\mathbf{x}_{bi}$ is the personal best value for particle $i$ and $\mathbf{x}_{bg}$ is the global best value among all the particles. $\mathbf{v}_i^j$ is the updated velocity of the $i^{th}$ particle in the $j^{th}$ iteration and $\mathbf{v}_i^{j-1}$ is the velocity value in the $(j-1)^{th}$ iteration. $\mathbf{x}_i^{j-1}$ is the position of the $i^{th}$ particle after the $(j-1)^{th}$ iteration.
\item For position update, the updated velocity is added to the existing position of the particle. The position update equation is \\
\begin{equation}
\mathbf{x}_i^j = \mathbf{x}_i^{j-1} + \mathbf{v}_i^{j} 
\end{equation}
\end{itemize}

\section{Proposed method}
Our interest lies in determining the shortest path between the two non-intersecting ellipsoids. The shortest path between the two ellipsoids in $\mathbb{R}^n$ can be formulated as a convex QCQP problem. 

\subsection{Formulation}
Arriving at the two end points of the shortest path, one on each ellipsoid, can be posed as minimization of $f(\mathbf{x})$ in a quadratic form and formulated as follows.
\begin{equation}
\begin{array}{c c c}
minimize & f(\mathbf{x})~=~(\mathbf{x}-\mathbf{y})^TP_0(\mathbf{x}-\mathbf{y}) \\ 
subject~to & \mathbf{x}^TP_1\mathbf{x} \leq 1 & \mathbf{x} \in X\\
 & \mathbf{y}^TP_2\mathbf{y} \leq 1 & \mathbf{y} \in Y
\end{array}
\end{equation}
where $\mathbf{x},\mathbf{y}~\in~\mathbb{R}^n$  and $X$,$Y$ are non-intersecting regions in $\mathbb{R}^n$, $X \cap Y = 0$. $P_1$ and $P_2$ are the matrices depicting the ellipsoids used in the optimization. 

In case the Euclidean distance metric is used for minimization of the path length, then $P_0$ is the identity matrix. The modified equation is,
\begin{equation}
\begin{array}{c c c}
minimize & f(\mathbf{x})~=~ \parallel \mathbf{x}-\mathbf{y} \parallel^2 \\ 
subject~to & \mathbf{x}^TP_1\mathbf{x} \leq 1 & \mathbf{x} \in X\\
 & \mathbf{y}^TP_2\mathbf{y} \leq 1 & \mathbf{y} \in Y
\end{array}
\end{equation}

Suppose one end point in the shortest path is known and fixed as $\mathbf{a}$. Then, Eq (5) can be reformulated as,
\begin{equation}
\begin{array}{c c c}
minimize & f(\mathbf{x})~=~ \parallel \mathbf{a}-\mathbf{x} \parallel^2 \\ 
subject~to & \mathbf{x}^TP_1\mathbf{x} \leq 1 & \mathbf{x} \in X\\
\end{array}
\end{equation}
 
Figure \ref{theoryexample} shows the ellipsoid with the covariance matrix $P_1$ with points $\mathbf{x}$ (inside), $\mathbf{x}_{B}$ (on the boundary) and $\mathbf{a}$ (outside). The boundary point $\mathbf{x}_{B}$ is nearest to the point $\mathbf{a}$ outside the ellipsoid. We need to determine the unknown $\mathbf{x}_{B}$ by the minimization of $f(\mathbf{x})$.


\subsection{Solution using PSO}
The novelty of this paper is the application of particle swarms for solving the QCQP problem. Particle swarms are deployed within the ellipsoid to determine $\mathbf{x}_{B}$. The function $f(\mathbf{x})$ is evaluated for each particle in the search space. The particle with the minimum $f(\mathbf{x})$ is considered as the closest to the point $\mathbf{a}$. Only one particle of the swarm is shown in Figure \ref{theoryexample} for ease of representation. 

\begin{figure}[!ht]
\centering
\includegraphics[width=12cm]{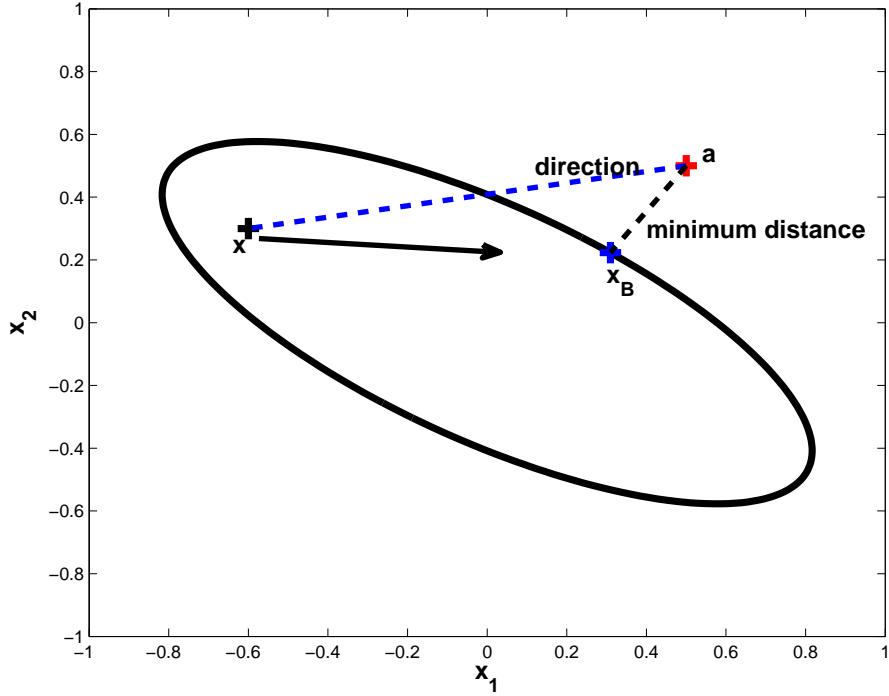}
\caption{An ellipsoid with a particle at a point $\mathbf{x}$ inside and a point $\mathbf{x}_B$ on its boundary nearest to the point $\mathbf{a}$ outside it (boundary point on the other ellipsoid). The dotted lines connecting the point $\mathbf{a}$ to the points $\mathbf{x}$ and $\mathbf{x}_B$ are shown. The desired direction of movement from $\mathbf{x}$ is also shown by an arrow, which is required to reach point $\mathbf{x}_B$.}
\label{theoryexample} 
\end{figure}

PSO is a stochastic evolutionary algorithm, which takes several generations to reach the optimal value and its performance depends on the initialization. The velocity update equation of the PSO algorithm is modified by including the function $f(\mathbf{x})$. The addition of evaluation function restricts the particle from moving away from the actual course towards the global best position. This addition provides an advantage in computation. However, it also constrains the particle to move in a particular direction. To counter this effect, we add a craziness term in the velocity update equation. The modified velocity update equation is:

\begin{equation}
\mathbf{v}_i^{j} = w\mathbf{v}_i^{j-1} + c_1r_1(\mathbf{x}_{bi} - \mathbf{x}_i^{j-1}) + c_2r_2(\mathbf{x}_{bg} - \mathbf{x}_i^{j-1}) + c_3r_3(\mathbf{a}-\mathbf{x}_i^{j-1}) + c_4r_4
\end{equation}
where $c_3$ and $c_4$ are constants, and $r_3$ is a random value that is varied during each iteration. The value of $c_3$ is chosen such that the term $(\mathbf{a}-\mathbf{x}_i^{j-1})$ does not take the particle outside the ellipsoid. $r_4$ is a random point on the surface of a sphere of dimension $n$ with randomly varying radius. $c_4r_4$ forms the craziness term.

\newtheorem{theorem}{Theorem}
\begin{theorem}
The velocity vector $\mathbf{v}$ of particle $\mathbf{x}$ should be directed towards the minimization of $f(\mathbf{x})$.
\end{theorem}

\begin{proof}
The function $f(\mathbf{x})~=~\parallel \mathbf{a}~-~\mathbf{x} \parallel^2$ needs to be evaluated. We use $\parallel \mathbf{a}~-~\mathbf{x} \parallel$
instead of using $\parallel \mathbf{a}~-~\mathbf{x} \parallel^2$ as the objective function. The value of $\mathbf{x}_{B}$ is unknown and needs to be arrived at by minimizing the function $f(\mathbf{x})$.






The value of the function for the $i^{th}$ particle at the $j^{th}$ iteration is evaluated as 
\begin{equation}
f(\mathbf{x}_i^j)~ = ~\parallel \mathbf{a}~-~\mathbf{x}_i^j \parallel
\end{equation}

Let the present position of the $i^{th}$ particle be split into the previous position and the velocity vector of the particle.

\begin{equation}
f(\mathbf{x}_i^j)~=~\parallel \mathbf{a}~-~\mathbf{x}_{i}^{j-1}~-~\mathbf{v}_{i}^{j} \parallel
\end{equation}

Here, $\mathbf{a}$ is fixed and the particle position $\mathbf{x}_{i}^{j-1}$ in the objective function is dependent on the velocity vector $\mathbf{v}_{i}^{j-1}$. So, one possible option for minimizing the function $f(\mathbf{x})$ is by changing the direction of the velocity vector of the particle as follows:

\begin{equation}
\mathbf{v}_{i}^{j} ~=~ \mathbf{a}~-~\mathbf{x}_{i}^{j-1}
\end{equation}

Thus, the function $f(\mathbf{x})$ is minimized if the direction of the velocity vector is as given by Eq (10). 
\end{proof}

The velocity vector update equation (2) does not have any term relating to minimization, but the modified Eq (7) includes the direction for minimization. It improves the convergence rate and thus reduces the computation time. The arrow (for representation purpose) shown in Figure \ref{theoryexample} is in the direction of minimizing the function $f(\mathbf{x})$ given by Eq (6).




On the assumption that one end point is known in the shortest path, particle swarms are placed in the search space of the other region and the other end point of the shortest path is determined. In order to evaluate the objective function in Eq (5), we need to determine one end point from region $X$ and the other from region $Y$. Two sets of particle swarms, one for each region, are placed within the search spaces of the respective regions. The objective function is evaluated based on the particles present in both the regions. In every iteration, the best position of a particle in one region is used as the known end point in the shortest path in the velocity update equation of the particles of the other region. The objective function reaches its optimal value after several generations.

In the process of optimization, some particles may often go out of the search space. To limit the particles within the search space, we inspect after every tentative position update as to whether the particle is lying within the search space by carrying out a check on the QCQP constraints. If the intended new position of a particle is going to violate the constraint, then its position is not updated. In other words, the particles likely to go out of the search space are redeployed back to their previous positions.

The proposed solution may be used in control system problems such as optimization of sensor networks or collaborative data mining, which are based on multiple agents or gradient projection \cite{zanella}. General consensus problem may be solved using the proposed method, where multiple agents need to reach a common goal \cite{matei,nedi2009,nedi2010}. Consensus or distributed optimization is discussed in \cite{boyd2011} as `consensus and sharing' using alternating direction method of multipliers (ADMoM). ADMoM uses one agent for each constraint. Such solvers are used for solving SVM in distributed format \cite{forero}.

\subsection{Algorithm}
Table \ref{patab1} presents a pseudo code for the algorithm. The parameters $w$,$~c_1$,$~c_2$,$~c_3$, and $~c_4$ are set to fixed values and the randomly varying parameters $r_1$,$~r_2$,$~r_3$, and $~r_4$ are updated in each iteration. The position, velocity, personal best, and global best of each particle are stored. The maximum number of iterations for the algorithm is specified as $T$ in the experiments and the size of swarm used in the algorithm is `10'. The proposed algorithm is implemented in MATLAB.

\begin{table}[!ht]
\caption{Proposed algorithm for QCQP using PSO}
\begin{center}
\begin{tabular}{l}
\hline
\textbf{Inputs:} $k=10$, $t_{max}=T$ and $f(\mathbf{x})$; set $w~=~ 0.05,~c_1~=~0.05,~c_2~=~0.05,~c_3~=$\\$0.05,~c_4~=~0.20$ and initialize parameters $\mathbf{x_i},~\mathbf{v_i}$	\\
\textbf{Outputs:} Global best value \\
\hspace{2mm}$t=0$,	\\
\hspace{2mm}\textbf{while} $t < t_{max}$	\\
\hspace{4mm}$t \leftarrow t+1$	\\
\hspace{4mm}\textbf{Function evaluation step:}	\\
\hspace{6mm}Calculate the function $f(\mathbf{x})$ for $\mathbf{x}_i$\\
\hspace{4mm}\textbf{Velocity update step:}	\\
\hspace{6mm}Randomly choose values for $r_1,~r_2,~r_3,~r_4$ in the range `0' and `1'.\\
\hspace{6mm}Then update the velocity of each particle as in Eq (7).\\
\hspace{4mm}\textbf{Position update step:}	\\
\hspace{6mm}Add updated velocity to existing position.\\
\hspace{6mm}Check constraint on the particles $\mathbf{x}^TP_1\mathbf{x}~\leq~1$.\\
\hspace{6mm}\textbf{for} $m$ = 1 to $k$	\\
\hspace{8mm}\textbf{if} $\mathbf{x}_m^TP_1\mathbf{x}_m~>~1$ \textbf{then}	\\
\hspace{10mm}$\mathbf{x}_m = \mathbf{x}_{m prev}$	\\
\hspace{8mm}\textbf{end if}	\\
\hspace{6mm}\textbf{end for}	\\
\hspace{2mm}\textbf{end while}	\\		
\hline
\end{tabular}
\end{center}
\label{patab1}
\end{table}

\section{Experiments and Results}
In this section, different optimization problems with quadratic constraints are solved using the proposed algorithm. The reliability of our algorithm is tested on linear and quadratic programming problems.

\subsection{Linear programming with quadratic constraints}
A LP problem with quadratic constraints is chosen in order to compare the convergence performance of our method with Khachiyan's and Karmarkar's methods. The LP problem is given below:
\begin{equation}
\begin{array}{c c c}
\max & f(x_1,x_2) = x_1 + x_2 & \\
subject~to & x_1^2 + x_2^2 \leq 1 & x_1,x_2 \in X
\end{array}
\end{equation}
where X is the region, satisfying the constraint.

This LP problem is reformulated as a QCQP problem:
\begin{equation}
\begin{array}{c c c}
\max & f(\mathbf{x}) = \mathbf{a}^T\mathbf{x} & \\
subject~to & \mathbf{x}^TA\mathbf{x} \leq 1 & \mathbf{x} \in X
\end{array}
\end{equation}
where vector $\mathbf{a}$ = $[1~1]^T$, $\mathbf{x}$ = $[x_1~x_2]^T$ and A is the identity matrix of size 2 (positive definite).  

\begin{figure}[!ht]
\centering
\includegraphics[width=12cm]{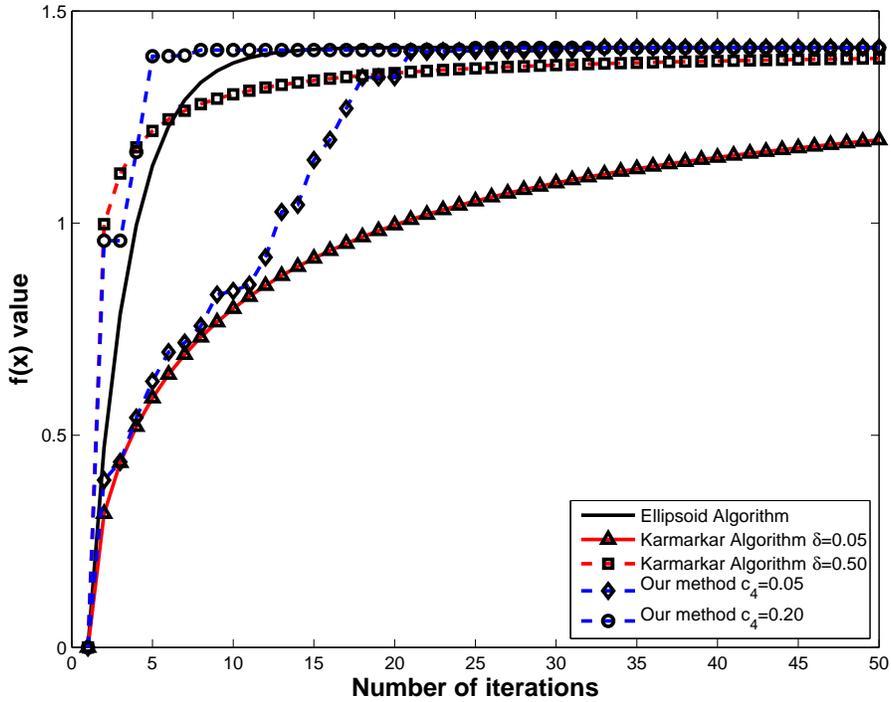}
\caption{A plot of evaluated value of $f(\mathbf{x})$  against the number of iterations for different algorithms for the LP problem given by Eq (11). At around 25 iterations, all the algorithms except Karmarkar's reach a value close to the solution. We observe that Karmarkar algorithm takes more iterations to reach the solution since the length of the direction vector decreases monotonically.}
\label{figurelp} 
\end{figure}

The solution for this LP problem is $x_1,~x_2~=~0.7071$ with $f(\mathbf{x})~=~1.4142$. This LP problem is solved using Khachiyan's ellipsoid algorithm, Karmarkar's algorithm and our method. Figure \ref{figurelp} shows the value of the function $f(\mathbf{x})$ for each iteration for a typical independent run of the experiment with 50 iterations. The length vector in Karmarkar's method is scaled by a variable $\delta$ \cite{karmarkar}. Two values are used for $\delta$ namely 0.05 and 0.50, for the evaluation of the function. As the value of $\delta$ increases, the algorithm reaches the optimal value of $f(\mathbf{x})$ in less number of iterations. Table \ref{across-methods} shows the error value of function $f(\mathbf{x})$ for all the algorithms after the $25^{th}$ iteration as shown in Figure \ref{figurelp}. The error values for the ellipsoid and our methods are less than $10$x$10^{-3}$.

\begin{table}[!ht]
\caption{Comparison of the error value of function $f(\mathbf{x})$ of our algorithm with those of two other methods for the LP problem given by Eq (11) after the $25^{th}$ iteration.}
\label{across-methods}
\centering
\begin{tabular}{|c|c|c|c|c|c|}
\hline
Algorithm $\rightarrow$ & Ellipsoid & Karmarkar  &  Karmarkar  & Our method & Our method \\
 & & $\delta=0.05$ & $\delta=0.50$ & $c_4=0.05$ & $c_4=0.20$ \\
\hline
Error value $\rightarrow$& 0.0001 & 0.3623 & 0.0494 & 0.0082 & 0.0016\\
\hline
\end{tabular}
\end{table}

In our algorithm, the values of the variables $w,~c_1,~c_2$, and $~c_3$ are set to $0.05$. These values are small so that the vector addition to position keeps the particles within the region X. Only the value of $c_4$ is varied since it scales the neighboring search space of a particle. The results for two different values of $c_4$ are shown in Figure \ref{figurelp}. 
We carried out 50 independent runs for this LP problem with two different values of $c_4$. The average, minimum, maximum and standard deviation of error value are tabulated in Table \ref{four-table}. The error values for $c_4=0.20$ are less, compared to $c_4=0.05$, indicating that as the neighboring search space is scaled to a higher value, the error value of the fitness function becomes better in fewer number of iterations.

\begin{table}[!ht]
\caption{The minimum, mean, maximum and standard deviation of error values for the function $f(\mathbf{x})$ over 50 independent runs of our algorithm for the LP problem defined by Eq (11).}
\label{four-table}
\centering
\begin{tabular}{|c|c|c|c|c|}
\hline
$c_4$ $\downarrow$ & Minimum & Mean & Maximum & Standard deviation\\
\hline
0.05 & 0.0020 & 0.0197 & 0.1346 & 0.0272 \\
\hline
0.20 & 0.0006 & 0.0182 & 0.0899 & 0.0208 \\
\hline
\end{tabular}
\end{table}

\subsection{Binary classification of simulated data}
Several variants of PSO have been applied for classification problems \cite{kennedybinclass}. Generally PSO is deployed in rule-based classification. A suitable classifier is chosen for classification; for example, a neural network \cite{kennedy}. The parameters like network weights are tuned using particle swarms. There are discrete versions of swarms, which can take a finite set of values \cite{cervantes}. Here, swarms learn the rule for classifying the test samples. In our method, binary classification is posed as QCQP and swarms are used to solve this QCQP problem. The input feature space is considered as the particle swarm space.  

\begin{figure}[!ht]
\centering
\includegraphics[width=12cm]{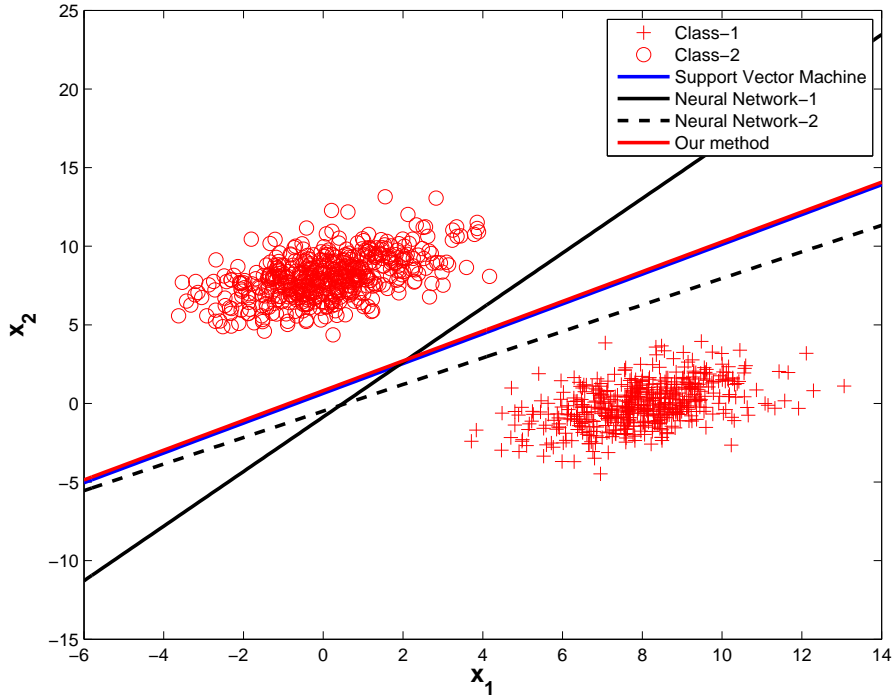}
\caption{Hyperplanes obtained by SVM, neural networks and our method are shown for a synthetic dataset for two classes. The hyperplane estimated by our method is closely aligned with that obtained by the SVM with a linear kernel. Other hyperplanes are arrived at by two neural networks (perceptron) and are not optimal.}
\label{figurebc} 
\end{figure}

 Figure \ref{figurebc} shows a simulated data set for two classes synthesized using two Gaussian distributions with means $\mathbf{\mu}_1~=~[8;0]; \mathbf{\mu}_2~=~[0;8]$ and the same covariance matrix $\Sigma_1, \Sigma_2~=~[2~1;1~2]$. A generic decision boundary for a binary classification problem is a hypersurface. The classes in the data are linearly separable thus reducing a hypersurface to a hyperplane. The equation of a hyperplane for this kind of data is given by,
\begin{equation}
z = w_1x_1 + w_2x_2 + w_0
\end{equation}
where $w_1$,$w_2$ are weights for the individual features and $w_0$ is the bias of the hyperplane. This hyperplane equation is used for classification:
\begin{equation}
\begin{array}{c}
{\normalsize If}~ z \geq 0, (x_1, x_2) \in \Omega_1\\
{\normalsize If}~ z < 0, (x_1, x_2) \in \Omega_2\\
\end{array}
\end{equation}
where $\Omega_1$ and $\Omega_2$ are the regions for the classes 1 and 2, respectively. 

The MATLAB programming platform is used to implement and test our method and also a SVM and neural network for comparison. We trained the SVM with a linear kernel and a neural network on this synthetic data. The hyperplanes learnt by the SVM and the neural network are shown in Figure \ref{figurebc}. A single layer perceptron algorithm with 100 epochs is used for training the neural network and SMO algorithm is used for training the SVM with linear kernel \cite{duda}.

We now explain the determination of the optimal hyperplane for binary classification by our method. The values of sample mean and covariance for the two classes are calculated from the simulated samples. Mahalanobis distance \cite{duda} is determined from the mean value of a class to the data points of the other class and the closest data point of the other class label is found. This closest point is used to fix the reference boundary of the search space (ellipsoid) of the other class. This process is repeated from the other class and the boundary of the first class is also determined. With the boundaries, ellipsoidal regions are formed with estimated means of classes as the centers. This is posed as a QCQP problem formulated in Eq (5) with the estimated covariance matrices normalized to boundary points as $P_1$ and $P_2$. Our algorithm is implemented by placing particle swarms near the mean value of the classes and evaluating the optimization function. The shortest path and the closest point on each boundary are estimated. 

\begin{equation}
\begin{array}{c c c}
minimize & (\mathbf{x}-\mathbf{y})^T(\mathbf{x}-\mathbf{y}) \\ 
subject~to & \mathbf{x}^T~E[(x_i-\mathbf{\mu}_1)(x_i-\mathbf{\mu}_1)^T]~~\mathbf{x}~\leq~1 & ~~\mathbf{x} \in X,~ x_i \in \Omega_1\\
 & \mathbf{y}^T~E[(y_i-\mathbf{\mu}_2)(y_i-\mathbf{\mu}_2)^T]~~\mathbf{y}~\leq~1 & ~~\mathbf{y} \in Y,~ y_i \in \Omega_2
\end{array}
\end{equation}

Eq (15) is optimized using the proposed algorithm. The intersection between the two ellipsoidal regions is assumed to be zero to eliminate a situation where particles reach to different solutions. Once the closest points on the boundaries are determined, the perpendicular bisector of the line joining the closest points is the hyperplane. The calculated hyperplane for binary classification is shown in Figure \ref{figurebc}. We can observe that the optimal hyperplane calculated by SMO algorithm and our method are closely placed, while other hyperplanes are not optimal. The estimated weight and bias values for each method are tabulated in Table \ref{table1}. The estimated weight values can be normalized as
\begin{equation}
\begin{array}{c}
w_1^1 = w_1~/~\sqrt{w_1^2 + w_2^2}\\
w_2^1 = w_2~/~\sqrt{w_1^2 + w_2^2}\\
\end{array}
\end{equation}
The normalized weights of SVM and our method are equal and they are $w_1^1=0.6875$ and $w_2^1=-0.7260$.

\begin{table}[!ht]
\caption{Weights and bias values calculated by SVM with a linear kernel, neural networks (perceptron) and our method for the binary classification problem with synthetic data.}
\label{table1}
\centering
\begin{tabular}{|c|c|c|c|}
\hline
Algorithm & $w_1$ & $w_2$ & $w_0$\\
\hline
Neural Network-1 & -13.2972 & 7.6540 & 6.5689\\
\hline
Neural Network-2 & -6.2830 & 7.4473 & 3.5954\\
\hline
SVM & 0.2718 & -0.2868 & 0.1838\\
\hline
Our method & 1.6697 & -1.7638 & 1.4376\\
\hline
\end{tabular}
\end{table}

\subsection{Performance on real datasets}
Our algorithm is also tested on some of the datasets available from UCI ML repository \cite{uciml}. The datasets used in our experiments are Iris, Wine, Pima and Thyroid. These four datasets have been chosen based on the consideration of minimal number of datasets with the maximum coverage of the different types of attributes (namely, binary, categorical value as integer, and real values). Further, the number of classes in each case is 2 or 3. So, the maximum number of hyperplanes to be obtained for any of these datasets is limited to three. The main characteristics of these datasets are tabulated in Table \ref{table-char}. The cross-validation performance on these datasets is compared with those of a SVM with linear and RBF kernel, a neural network, and GSVM. GSVM \cite{gonzalez} estimates a classification hyperplane similar to the proposed approach and is developed on the fundamentals of optimizing the SVM problem. GSVM modifies the bias ($w_0$) obtained from SVM by moving the hyperplane such that the geometric mean of recall and precision is improved.

\begin{table}[!ht]
\caption{The characteristics of datasets chosen for experimentation from UCI ML repository \cite{uciml} with varied nature of attributes.}
\label{table-char}
\centering
\begin{tabular}{|c|c|c|c|c|}
\hline
Dataset & No. of & No. of & No. of & Nature of\\
 & samples & Attributes & classes & attributes\\
\hline
Iris & 150 & 4 & 3 & real\\
\hline
Wine & 178 & 13 & 3 & real \& integer\\
\hline
Pima & 768 & 8 & 2 & real \& integer\\
\hline
Thyroid & 215 & 5 & 3 & real \& binary\\
\hline
\end{tabular}
\end{table}

\subsubsection{Iris Dataset}
The Iris dataset consists of four different measurements on 150 samples of iris flower. There are 50 samples of three different species of iris forming the dataset. The features, whose values are available from the dataset, are length and width of leaves and petals of different iris plants. Out of the three species, two are not linearly separable from each other, whereas the third is linearly separable from the rest of the species. The classification task is to determine the species of the iris plant, given the 4-dimensional feature vector.

\subsubsection{Wine Dataset}
The Wine dataset contains the different physical and chemical properties of three different types of wines derived from three different strains of plants. Some of the physical properties such as hue and colour intensity have integer values, whereas chemical properties such as ash or phenol content have real values. The feature vector has 13 dimensions and there are a total of 178 samples. The classification task is to determine the type of wine, given the values of the content of the thirteen physical and chemical components.

\subsubsection{Pima Indians Diabetes Dataset}
The Pima dataset contains eight different parameters measured from 768 adult females of Pima Indian heritage. Once again, some of them are integer valued, such as age and number of pregnancies. Certain other parameters, such as serum insulin, are real valued. It is a two-class classification problem of identifying normal and diabetic subjects, given the 8-dimensional feature vector as input.

\subsubsection{Thyroid Disease Dataset}
This dataset contains ten distinct databases of different dimensions. The particular database chosen for our study contains 5 different parameters measured from 215 individuals. Some of the variables have binary values, while others have real values. The classification task is to assign an individual to one of 3 classes, given the 5-dimensional feature vector as input.

\subsubsection{Data projection}
\label{dataproj}
We perform a preprocessing step on these real datasets. This step is necessary since,
\begin{itemize}
\item Some of the attributes of the dataset have larger variance than others. This may result in skewed ellipsoid formation at the mean value of the dataset.
\item The number of samples available in the datasets is also less.
\end{itemize}

To overcome these two problems, we perform the two steps given below.
\begin{itemize}
\item Eigen value decomposition is performed on dataset covariance matrix. The dataset is projected on to these eigen vectors. Scaling is performed in such a way that each component in the new projected dataset has unit variance.  
\item This step is performed independently on each of the subsets used in the cross-validation stage. We project the subset of samples into two dimensions. First is the direction of the vector joining the sample means of the classes. The equation for this vector is,
\begin{equation}
p~\alpha~(\mathbf{\mu}_1~-~\mathbf{\mu}_2) 
\end{equation}
where $p$ is the projection vector, $\mathbf{\mu}_1$ and $\mathbf{\mu}_2$ are the sample means of classes-1 and 2, respectively. The second direction is that of the eigen vector corresponding to the largest eigen value of the covariance matrix of the subset. 
\end{itemize}

The projected samples are used in the estimation of new sample means and covariance matrices. The estimated hyperplane is used to classify the test samples. 

\subsubsection{Cross-validation}
These datasets do not have separate training and test samples; hence we perform cross-validation. In cross-validation, a small subset is used for testing while the remaining are used as training samples. We use ten fold cross-validation: split the datasets into ten subsets and use one of them for testing and the others for training at a time and then rotate the subsets. Equation (15) is used in the estimation of the hyperplane, which in turn is used to classify the test samples. Ten trials are performed and the average cross-validation errors are reported in Table 6. 

\begin{table}[!ht]
\caption{Cross-validation (CV) error (in \%) using our method for different datasets from UCI ML repository \cite{uciml} compared with those reported in \cite{gonzalez} and our own implementation of SVM and neural network.}
\label{table2}
\centering
\begin{tabular}{|c|c|c|c|c|c|}
\hline
Dataset & CV error by & CV error & CV error in & CV error in & CV error\\
 & Our method & in SVM with & Neural & SVM with RBF & in GSVM\\
 & & linear kernel & Network & kernel \cite{gonzalez} & \cite{gonzalez}\\
\hline
Iris & \textbf{2.20} & 3.20 & 3.60 & 4.21 & 3.92\\
\hline
Wine & 1.06 & 1.22 & 1.56 & 1.60 & \textbf{0.93}\\
\hline
Pima & 25.15 & \textbf{23.06} & 64.78 & 24.64 & 25.85\\
\hline
Thyroid & 4.32 & 3.04 & 9.20 & 2.05 & \textbf{1.73}\\
\hline
\end{tabular}
\end{table}

The performance of our method is close to that of the variants of SVM and is superior to that of the neural network. We notice that in Iris and Wine datasets, the cross-validation error obtained by our technique is better than those achieved by SVM with linear and RBF kernels and also the neural network. In Pima and Thyroid datasets, the cross-validation error is high, due to the high degree of correlation. 

\section{Conclusion and Future work}
We have developed a classification method and optimization algorithm for solving QCQP problems. The novelty in this method is the application of particle swarms, an evolutionary technique, for optimization. The results indicate that our approach is a possible method in solving general QCQP problems without gradient estimation. We have shown the results of our algorithm under quadratic constraints by evaluating different optimization functions. The issue with PSO based methods is their computational complexity and the need for parameter tuning. The number of function evaluations linearly increases with the number of particles employed and the number of iterations carried out.

In future, we intend to learn multiple hyperplanes by placing multiple kernels in each class and evaluating the performance against multiple-kernel learning algorithms. The hyperplanes estimated for different kernels may reduce the cross-validation error for the Pima and Thyroid datasets.




\end{document}